\RequirePackage{amsmath}
\documentclass[american,a4paper,runningheads,envcountsect]{llncs}
 
\usepackage{babel}
\usepackage[utf8]{inputenc}
\usepackage[T1]{fontenc}

\usepackage{amsmath}
\usepackage{amsthm}
\usepackage{amssymb}
\usepackage{mathtools}
\usepackage{mathrsfs}
\usepackage{stmaryrd} 
\usepackage{dsfont} 

\usepackage{graphicx}
\usepackage{float} 

\usepackage{subcaption}
\usepackage[ruled,vlined,linesnumbered]{algorithm2e}

\usepackage{tikz}
\usetikzlibrary{shapes,arrows,positioning,calc} 
\usetikzlibrary{matrix}
\usepackage{tikzscale}

\usepackage{csquotes}
\usepackage{booktabs}
\usepackage{enumitem}

\usepackage{fca_mod}
 
\usepackage[colorlinks,citecolor=blue,linkcolor=blue,urlcolor=black,linktocpage]{hyperref}
\hypersetup{
  linktoc=all,
  pdftitle={Attribut Exploration with Multiple Contradicting Partial Experts},
  pdfauthor={Maximilian Felde and Gerd Stumme},
  pdfkeywords={Formal Concept Analysis, Attribute Exploration, Incomplete Information, Multiple Experts},
  bookmarksnumbered=true,     
  bookmarksopen=true,         
  bookmarksopenlevel=1,       
  pdfstartview=Fit,           
  pdfpagemode=UseOutlines,    
}
\usepackage[all]{hypcap}

\usepackage{cite}

\usepackage{cleveref}
\let\cref\Cref
\crefname{strategy}{Strategy}{Strategies}
\crefname{defn}{Definition}{Definitions}
\Crefname{defn}{Definition}{Definitions}

\theoremstyle{definition}

\theoremstyle{remark}

\newcommand{\ie}{i.\,e.\xspace}

\newcommand{\tif}{\text{ if }}

\renewcommand{\GMI}[1][]{{(G,M,\relI_{#1})}}
\newcommand{\GMII}[1][]{(G_{#1},M,\relI_{#1})}

\newcommand{\x}{\times}
\renewcommand{\o}{o}
\newcommand{\q}{?}
\newcommand{\values}{\ensuremath{\{\x,\o,\q\}}}
\newcommand{\GMWI}[1][]{(G,M,W,\relI_{#1})} 
\newcommand{\GMWII}[1][]{(G_{#1},M,W,\relI_{#1})} 

\newcommand{\K}{\context}

\newcommand{\C}{\context[C]}

\renewcommand{\L}{\mathcal{L}}
\renewcommand{\implies}{\rightarrow}

\newcommand{\Mod}{\operatorname{Mod}} 

\newcommand{\Impm}{\ensuremath{\operatorname{Imp}_M}}
\newcommand{\Imp}[1][\context]{\operatorname{Imp}(#1)}
\newcommand{\Sat}[1][\context]{\operatorname{Sat}(#1)}

\newcommand{\experts}{\mathcal{E}}

\newcommand{\view}[1][]{(\context_{#1},\mathcal{L}_{#1})}

\renewcommand{\epsilon}{\varepsilon}
\renewcommand{\phi}{\varphi}

\usepackage{datetime}

\usepackage{stackengine}

\newcommand{\customrule}{\raisebox{0.1pt}[0.2pt]{\rule{7pt}{0.1pt}}}
\newcommand{\subposz}{
  \setstackgap{L}{0.5pt}
  \setstackgap{S}{1pt}
  \Shortstack[c]{. {\customrule} . {\customrule} .}%
}

\newcommand*{\bigsubpos}{\subposz}
\newcommand{\modeledby}{\reflectbox{$\models$}}
\newcommand{\U}{\mathbb{U}}
\newcommand{\Lkrel}{\L_{\K\L_{0}}}
\newcommand{\Lurel}{\L_{\U\L_{0}}}
\DeclareUnicodeCharacter{27F6}{$\implies$} 
\usepackage[outline]{contour} 
\contourlength{2.5pt}

\usepackage{placeins} 
\usepackage{changepage} 

\begin{document}

\title{Attribute Exploration with Multiple Contradicting Partial Experts}
\subtitle{}

\author{Maximilian Felde\inst{1,2}\orcidID{0000-0002-6253-9007} \\\and Gerd Stumme\inst{1,2}\orcidID{0000-0002-0570-7908}}

\date{\today} 
 
\institute{%
  Knowledge \& Data Engineering Group,
  University of Kassel, Germany\\[0.5ex]
  \and
  Interdisciplinary Research Center for Information System Design\\
  University of Kassel, Germany\\[0.5ex]
  \email{felde@cs.uni-kassel.de, stumme@cs.uni-kassel.de}
}

 

\maketitle

\begin{abstract}
  Attribute exploration is a method from Formal Concept Analysis (FCA) that helps a domain expert discover structural dependencies in knowledge domains which can be represented as formal contexts (cross tables of objects and attributes).
  In this paper we present an extension of attribute exploration that allows for a group of domain experts and explores their shared views.
  Each expert has their own view of the domain and the views of multiple experts may contain contradicting information.

\end{abstract}

 
\section{Introduction}
\label{sec:introduction}
 
Attribute exploration\cite{ganter1984two} is a well known knowledge acquisition method from Formal Concept Analysis (FCA)\cite{GanterWille1999}. In domains that can be represented as formal contexts (binary tables that capture the relation between objects and attributes), attribute exploration allows a domain expert to efficiently discover all attribute dependencies in the domain.

The basic idea of attribute exploration is to extend domain information using a question-answering scheme with answers provided by a domain expert.
The questions are posed in the form of implications over the attributes of the domain.
The attribute exploration algorithm determines the next question $A_1\ldots A_n\implies B?$ (read: Do attributes $A_1\ldots A_n$ imply attribute $B$ in the domain?) which the expert then either confirms or refutes.
If the expert refutes the implication then a counterexample has to be provided, \ie, an object from the domain that has the attributes $A_1\ldots A_n$ but lacks attribute $B$.
The algorithm poses these questions
until the validity of every conceivable implication can be inferred from the obtained domain information, \ie, every implication either follows from the accepted implications or one of the examples serves as counterexample.

In the basic setting, attribute exploration requires a single all-knowing expert.
Extensions that allow for background information\cite{stumme96attribute,ganter1999attribute} and for an expert with partial knowledge\cite{conf/iccs/BurmeisterH00,holzer2001dissertation,holzer2004knowledgeP1,holzer2004knowledgeP2} have been developed.
However, these are just two examples from a vast body of extensions and variations that have been studied.
An extensive overview can be found in the book \emph{Conceptual Exploration}\cite{ganter2016conceptual}.

Despite the many advances to attribute exploration, its extension to multiple experts has only recently gained interest, cf. \cite{conf/iccs/HanikaZ18,DBLP:journals/corr/abs-1908-08740,conf/iccs/Kriegel16,FeldeStumme2020TriadicExploration}.
And, except for \cite{FeldeStumme2020TriadicExploration}, the basic assumption made is that some true universe exists (for a given domain).

In this paper we generalize the approach of \cite{FeldeStumme2020TriadicExploration} and suggest a framework for attribute exploration with multiple experts without the assumption of the existence of some true universe.
Here, multiple experts can have contradicting views of a domain.
Contradicting views occur naturally in domains that are subjective (e.g. opinions).
Our approach also allows for the situations where some true universe exists but the experts may be imperfect and make mistakes when answering questions.
Our aim is to identify both the common and the conflicting views.
The resolution of those conflicts is out of the scope of~our~approach.
 

The paper is structured as follows:
In \cref{sec:foundations}, we recollect the basics of FCA 
and 
how to model incomplete information in FCA
via incomplete contexts, possible and certain derivations, information order and satisfiable implications.
In the main section, \cref{sec:exploration-with-multiple-experts}, we first discuss the problem of attribute exploration with multiple experts.
Then, we introduce a representation of expert views of a domain and the notion of shared implications. We provide an algorithm to explore the shared implication theory for a group of experts and discuss how to explore the shared implication theories for some or all subsets of a group of experts.
Afterwards, we discuss how this approach can be used as a first step in a collaborative exploration to find a common view and identify the conflicts that exist among the group of experts.
Finally, we provide conclusion and outlook in \cref{sec:conclusion-outlook}.
In order to keep the paper as concise as possible we provide an example only in the \hyperref[appendix]{appendix}.
Note that we do not provide a separate section for related work and instead combine it with the recollection of the basics and provide context on related work where appropriate.

\section{Foundations}\label{sec:foundations}
FCA was introduced by Wille in \cite{Wille82}.
As the theory matured, Ganter and Wille compiled the mathematical
foundations of the theory in \cite{GanterWille1999}.
We begin by recalling some notions from FCA and some extensions that allow us to model incomplete and conditional information.

\subsection{Formal Concept Analysis}
%
%
\subsubsection{Formal Context and Concept.}
A \emph{formal context} $\context = \GMI$ consists of a set $G$ of objects, a set $M$ of attributes and an incidence relation
$I\subseteq G\times M$ with $(g,m)\in I$ meaning ``\emph{object $g$
  has attribute $m$}''.
There are several interpretations for $(g,m)\not\in I$,
cf. \cite{burmeister1991merkmalimplikationen,conf/iccs/BurmeisterH00},
the standard one being, ``\emph{$g$ does not have the attribute $m$ or
  it is irrelevant, whether $g$ has $m$}''.  In the following we
interpret $(g,m)\not\in I$ as ``\emph{$g$ does not have $m$}'', which
is reasonable when modeling incomplete knowledge.

This interpretation can be equivalently modeled by a (two-valued)
\emph{formal context} $\context = \GMI$ that consists of a set of
objects $G$, a set of attributes $M$ and an \emph{incidence function}
$I\colon G\times M \to \{\x,\o\}$.  The incidence function describes whether
an object $g$ has an attribute $m$: $I(g,m)=\x$ means ``\emph{$g$ has
  $m$}'' and $I(g,m)=\o$ means ``\emph{$g$ does not have $m$}''.\footnote{For modeling incomplete information, we later extend $\{\x,\o\}$ by ``\q'', cf. \cref{def:incomplete-context}.}
Clearly we can use a one-valued formal context to define an equivalent
two-valued formal context and vice versa using
$(g,m)\in I \Leftrightarrow I(g,m)=\x$.  In the following we will use
these two notations interchangeably.
Two derivation operators
$(\cdot)'\colon\mathcal{P}(M)\rightarrow \mathcal{P}(G)$ and
$(\cdot)'\colon\mathcal{P}(G)\rightarrow \mathcal{P}(M)$ are defined in the following
way: For a set of objects
$A\subseteq G$, the set of \emph{attributes common to the
  objects in $A$} is provided by
  $A' \coloneqq \{ m \in M \mid \forall g\in A: (g,m)\in I \}.$
Analogously, for a set of attributes $B\subseteq M$, the set
of \emph{objects that have all the attributes from $B$} is provided by
$B' \coloneqq \{ g \in G \mid \forall m\in B: (g,m)\in I \}$.
To prevent ambiguity if we have multiple contexts, we use the incidence relation $I$ and write $A^I$ instead of $A'$.
For a single object $g\in G$ (or attribute) we omit the parentheses and simply write $g'$.
A \emph{formal concept} of a formal context $\context = \GMI$ is a
pair $(A,B)$ with $A\subseteq G$ and $B\subseteq M$ such that $A'=B$
and $A=B'$.  $A$ is called the \emph{extent} and $B$ the \emph{intent}
of the formal concept $(A,B)$.  The set of all formal concepts of a
context $\context$ is denoted by $\mathfrak{B}(\context)$.
Note that for any set $A \subseteq G$ the set $A'$ is the intent of a
concept and for any set $B\subseteq M$ the set $B'$ is the extent of a
concept.
The subconcept-superconcept relation on $\mathfrak{B}(\context)$ is
formalized by:
$(A_1,B_1)\leq(A_2,B_2) :\Leftrightarrow A_1\subseteq A_2
(\Leftrightarrow B_1 \supseteq B_2)$.
The set of concepts together with this order relation
$(\mathfrak{B}(\context),\leq)$, also denoted $\BV(\K)$, forms a complete lattice, the
\emph{concept lattice}.
The vertical combination of two formal contexts
$\context_i=\GMII[i], i\in\{1,2\}$ on the same set of attributes $M$
is called the \emph{subposition} of $\context_1$ and
$\context_2$ which is denoted by $\frac{\K_1}{\K_2}$.
Formally, it is defined as
$(\dot{G}_1 \cup \dot{G}_2,M,\dot{I}_1\cup \dot{I}_2)$, where
$\dot{G}_i:=\{i\}\times G$ and
$\dot{I}_i:= \{((i,g),m)|(g,m)\in I_i\}$ for $i\in \{1,2\}$.
The \emph{subposition} of a set of contexts $\{\K_1,\ldots,\K_n\}$ on the same set of attributes is defined analogously and we denote this by $\bigsubpos_{i\in\{1,\ldots,n\}} \K_i$. 
\subsubsection{Attribute Implications.}

Let $M$ be a finite set of attributes.
An \emph{attribute implication} over $M$ is a pair of subsets
$A,B\subseteq M$, denoted by $A\implies B$, cf. \cite{GanterWille1999}.
$A$ is called the \emph{premise} and $B$ the \emph{conclusion} of the implication $A\implies B$.
The set of all implications over a set $M$ is denoted by $\Impm  = \{A\implies B| A,B \subseteq M\}$.

A subset $T\subseteq M$ \emph{respects} an attribute implication $A\implies B$ over $M$ if $A\not\subseteq T$ or $B\subseteq T$. Such a set $T$ is also called a \emph{model} of the implication.
$T$ \emph{respects a set} $\L$ of implications if $T$ respects all implications in $\L$.
An implication $A\implies B$ \emph{holds} in a set of subsets of $M$ if each of these subsets respects the implication.
$\Mod{\L}$ denotes the set of all attribute sets that respect a set of implications $\L$; it is a closure system on $M$. The respective closure operator is denoted with $\L(\cdot)$.

For a formal context $\context=\GMI$ an implication $A\implies B$ over $M$ \emph{holds in the context} if for every object $g\in G$ the object intent $g'$ respects the implication.
Such an implication $A\implies B$ is also called a \emph{valid implication} of $\context$.
Further, an implication $A\implies B$ holds in $\context$ if and only if $B\subseteq A''$, or equivalently $A' \subseteq B'$.
The set of all implications that hold in a formal context $\K$ is denoted by $\Imp[\K]$.
An implication $A\implies B$ \emph{follows} from a set $\L$ of implications over $M$ if each subset of $M$ respecting $\L$ also respects $A\implies B$.  A family of implications is called \emph{closed} if every implication following from $\L$ is already contained in $\L$.  Closed sets of implications are also called \emph{implication theories}.
For an implication theory $\L$ on $M$, the context $\K=(\Mod\L,M,\ni)$ is a context such that $\Imp[\K]=\L$  and $\Mod\L$ is the system of all concept intents.
%
%
Because the number of implications that hold in some context $\context$ can be very large, one usually works with a subset $\L$ of implications that is \emph{sound} (implications in $\L$ hold in $\K$), \emph{complete} (implications that hold in $\K$ follow from $\L$) and \emph{irredundant} (no implication in $\L$ follows from other implications in $\L$).
Such a subset of implications is also called \emph{implication base}.
A specific base with minimal size is the so-called \emph{canonical base}, cf. \cite{GanterWille1999,guigues1986familles}.

\subsubsection{Attribute Exploration.}
Attribute exploration is a method to uncover the implication theory for a domain with the help of a domain expert. 
It works by asking the expert questions about the validity of implications in the domain.
The basic approach~\cite{ganter1984two} computes the canonical base. This requires the expert to have a complete view of the domain and respond to a question \emph{does $R\implies S$ hold?} by either confirming that the implication holds in the domain or by rejecting that it holds with a counterexample, i.e., an object $g$ from the domain together with all its relations to the attributes such that $R\subseteq g'$ but $S\not\subseteq g'$. The exploration algorithm successively poses questions until all implications can either be inferred to follow from the set of accepted implications or be rejected based on some counterexample. More precisely, the questions are usually asked in lectic order with respect to the premise $R$ (which can be computed with the \emph{NextClosure} algorithm, c.f. \cite{ganter1984two,GanterWille1999,ganter2016conceptual}) and the conclusion $S$ is the largest set of attributes that can follow from $R$ with respect to the examples provided so far.

Many extensions and variants of attribute exploration have been developed since its introduction in~\cite{ganter1984two}.
Some examples are the use of background information and exceptions studied by Stumme~\cite{stumme96attribute} and Ganter~\cite{ganter1999attribute} and the use of incomplete information studied by Holzer and Burmeister~\cite{burmeister1991merkmalimplikationen,conf/iccs/BurmeisterH00,holzer2001dissertation,holzer2004knowledgeP1,holzer2004knowledgeP2}.
A good overview can be found in 
\emph{Conceptual Exploration}~\cite{ganter2016conceptual} by Ganter and Obiedkov.

\subsubsection{Relative Canonical Base.}

The canonical base has been generalized to allow for (background) implications \cite{stumme96attribute} as prior information.
We use this extensively in \Cref{subsec:explore-shared-implications} when we explore all shared implications for a group of experts.

Assuming we have a formal context $\context=\GMI$ and a set of (background) implications $\mathcal{L}_0$ on $M$ that hold in the context $\context$, a \emph{pseudo-intent} of $\context$ \emph{relative to $\L_0$}, or \emph{$\L_0$-pseudo-intent}, is a set $P\subseteq M$
with the properties:
\begin{enumerate}
\item $P$ respects $\L_0$
\item $P\not = P''$
\item If $Q\subseteq P,\ Q \not = P$, is an $\L_0$-pseudo-intent of $\K$ then $Q''\subseteq P$.
\end{enumerate}
The set $\Lkrel := \{P\implies P''| P \text{ an }\L_{0}\text{-pseudo-intent of } \K\}$ is called the \emph{canonical base} of $\K$ \emph{relative to $\L_0$}, or simply the \emph{relative canonical base}.
Note that all implications in $\Lkrel$ hold in $\K$.
Further (see \cite{GanterObiedkov04Triadic,stumme96attribute}), if all implications of $\L_0$ hold in $\K$, then
  \begin{enumerate}
  \item each implication that holds in $\K$ follows from $\Lkrel \cup \L_0$, and
  \item $\Lkrel$ is irredundant w.r.t. 1.
  \end{enumerate}

\subsection{Incomplete Information in FCA}
In order to model partial information about a domain we use the notion of an \emph{incomplete context}, a special multi-valued context, which can be interpreted as a formal context with some missing information. This notion of partial information has been extensively studied by Burmeister and Holzer, c.f. \cite{conf/iccs/BurmeisterH00,holzer2001dissertation,holzer2004knowledgeP1,holzer2004knowledgeP2}.
We now recollect some notions.

\subsubsection{Incomplete Context.}
\begin{definition}\label{def:incomplete-context}
  An \emph{incomplete context} is defined as a three-valued context
  $\context = \GMWI$ consisting of a set of objects $G$, a set of
  attributes $M$, a set of values $W=\values$ and an incidence function
  $I\colon G\times M \to \values$.  For $g\in G$ and $m\in M$ we say that
  ``\emph{$g$ has $m$}'' if $I(g,m)=\times$,
  ``\emph{$g$ does not have $m$}'' if $I(g,m)=\o$ and
  ``\emph{it is not known whether $g$ has $m$}'' if $I(g,m)=\q$.
\end{definition}

  
Another possibility to model incomplete information about the relation between objects and attributes is to use a pair of formal contexts $(\K_+,\K_{?})$ on the same sets of objects and attributes such that $\K_+$ models the attributes that the objects certainly have and $\K_{?}$ models the attributes that the objects might have, c.f.~\cite{ganter2016conceptual,ganter1999attribute}. This is equivalent to our representation as incomplete context.


The subposition for incomplete contexts $\context_i=\GMWII[i], i\in\{1,2\}$ is the incomplete context $\frac{\K_1}{\K_2} = (\dot{G}_1 \cup \dot{G}_2,M,W,I)$ where $I((i,g),m) = I_i(g,m)$.
The \emph{subposition} of a set of incomplete contexts $\{\K_1,\ldots,\K_n\}$ on the same set of attributes is defined analogously and denoted by $\bigsubpos_{i\in\{1,\ldots,n\}} \K_i$.

\subsubsection{Information Order.}
On the values $\values$ of incomplete contexts we define the \emph{information order $\leq$} where $\q \leq \x$, $\q \leq \o$ and $\x$ and $\o$ are incomparable, c.f. \cite{Belnap1977,fitting1991kleene,journals/ci/Ginsberg88a,holzer2001dissertation}.
This order is used to compare different incomplete contexts on the same set of attributes in the following way:
Given two incomplete contexts $\K_1=\GMWII[1]$ and $\K_2=\GMWII[2]$ we say that $\K_2$ contains at least as much information as $\K_1$, denoted $\K_1\leq \K_2$ if $G_1\subseteq G_2$ and for all $(g,m)\in G_1\times M$ we have $I_1(g,m)\leq I_2(g,m)$ in the information order.

For two incomplete contexts $\K_1$ and $\K_2$ on the same set of attributes $M$ that have no conflicting information, \ie, when there is no $g\in G_1\cap G_2$, $m\in M$ with $I_1(g,m)$ and $I_2(g,m)$ incomparable, their supremum $\K_1\vee \K_2 := (G_1\cup G_2,M,W,I)$ is obtained by defining $I(g,m)$ as supremum of $I_1(g,m)$ and $I_2(g,m)$ where $I_i(g,m) := \ \q \tif g\not\in G_i$.

Incomplete contexts that contain no ``$\q$'' are identified with the respective formal context.
A \emph{completion} of an incomplete context $\K$ is a formal context $\hat\K$ with $\K \leq \hat\K$.

\subsubsection{Derivation Operators for Incomplete Contexts.}
The derivation operator $\cdot'$ is only defined for formal contexts, however, there are some analogue operators for incomplete contexts that capture the notions of possible and certain relations.

  Given an incomplete context $\context = \GMWI$ the
  \emph{certain intent} for $A\subseteq G$ is defined by
  \(
    A^\Box \coloneqq \{ m \in M \mid I(g,m)=\x \text{ for all } g\in A
    \}
  \)
  and the \emph{possible intent} by
  \(
    A^\Diamond \coloneqq  \{ m \in M \mid I(g,m)\neq \o \text{ for all } g\in A \}.
  \)
  For $B\subseteq M$ the \emph{certain extent} $B^\Box$ and
  the \emph{possible extent} $B^\Diamond$ are defined analogously.
  For $g\in G$ and $m\in M$ we use the abbreviations $g^\Box$,
  $g^\Diamond$, $m^\Box$ and $m^\Diamond$.
%
For a incomplete contexts without ``$?$'' or formal context the certain and possible derivations and the usual derivation operator $\cdot'$ are the same, i.e., $A^{\Box}=A^{\Diamond}=A'$.

\subsubsection{Implications and Incomplete Contexts.}
We now recollect a notion of satisfiable implications in an incomplete context $\K$. Satisfiable implications of $\K$ are implications that have no counterexample in $\K$ and where a completion $\hat\K$ of $\K$ exists in which the implication holds.
Formally, an attribute implication $R\implies S$ is \emph{satisfiable} in an incomplete context if
$S\subseteq R^{\Box\Diamond}$ (equivalent: if $\forall g\in G: R\subseteq g^{\Box} \Rightarrow S \subseteq g^{\Diamond}$).
For a premise $R$, the maximal satisfiable conclusion is $R^{\Box\Diamond}$.
It is the largest set of attributes that all objects, that certainly have all attributes in $R$, possibly have.
The set of all satisfiable implications in an incomplete context $\K$ is denoted $\Sat[\K]$.

\begin{definition}
  Let $\K=\GMWI$ be an incomplete context.
  Given a set of satisfiable implications $\L$, i.e., $\L\subseteq \Sat[\K]$,  we define the \emph{$\L$-completion} of $\K$ as the formal context $\overline\K={(\overline G,M,J)}$ that is obtained by letting $g^J=\L(g^{\Box})$ for all $g\in G$ and adding a new object $h$ to $\overline G$ for each model $B\in\Mod \L$  that is not yet an intent of an object in $G$ such that $h^J=B$.
\end{definition}
\begin{lemma}\label{lem:L-closed-context-has-L-implications}
   Let $\L$ be a closed set of implications that is satisfiable in an incomplete context $\K$. Then $\K\leq \overline\K$ and $\Imp[\overline\K]=\L$.
 \end{lemma}
 \begin{proof}
   Because $\Imp[(\Mod\L,M,\ni)]=\L$, by construction of $\overline \K$ it follows that $\Imp[\overline\K]=\L$.
   Further, for a closed and satisfiable set of implications $\L$ we have $\L(g^{\Box})\subseteq g^{\Diamond}$ for all $g\in G$ and thus $\K\leq \overline\K$.
 \end{proof}

 

\section{Exploration with Multiple Experts}\label{sec:exploration-with-multiple-experts}
In this section, we examine the problem of attribute exploration with multiple experts.
We introduce a formal representation of expert views and shared implications and study some of their properties.
We provide an extension to the attribute exploration algorithm that allows for exploration of the shared implications of a group of experts and prove its correctness.
In addition, we discuss the exploration of shared implications for some or all subsets of a group of experts.

Despite the development of many variants of and extensions to attribute exploration, the inclusion of multiple experts has only recently started to get some attention, c.f.\cite{FeldeStumme2020TriadicExploration,conf/iccs/HanikaZ18,conf/iccs/Kriegel16,DBLP:journals/corr/abs-1908-08740}.
However, in most of these works the underlying assumption is that there exists some true universe (for a given domain), \ie, there is no disagreement between the experts.
The exception is \cite{FeldeStumme2020TriadicExploration}, which introduces a triadic approach to attribute exploration and proposes a possible adaption to deal with multiple experts.

In this paper we expand on the ideas from \cite{FeldeStumme2020TriadicExploration} and give them a more general theoretical framework.
In our setting we do not assume that the views of multiple experts must be compatible, i.e., two experts can disagree whether an object has an attribute or whether an implication holds in the domain.
This can be the case when we have domains without an objectively \emph{correct} view, for example, when dealing with opinions or other subjective properties.
Non-compatible views can also arise, because real experts are normally not perfect and, even if there exists an objectively correct view of a domain, some experts might make mistakes, draw wrong conclusions or have some false ideas about the domain.
The result, however, is always the same:

When we ask multiple experts, the answers we get can contain incompatible information.
And, any attempt to combine the information to produce a single consistent answer to conform to the classic attribute exploration approach means we throw away part of the information and introduce some artifacts.
When we talk about opinions the idea of combining opposing opinions clearly makes no sense.
But, even if we know there exists some correct view of a domain, and we received answers with incompatible information it is usually impossible to deduce which information is correct and which is not.
Hence, any resolution method without access to the correct information (assuming it exists) is bound to introduce information artifacts when attempting to merge incompatible answers.

Thus, in these cases it does not make sense to try and combine the information provided by multiple experts in order to obtain a more detailed view of the domain.
Instead, we can try to find the parts of the domain where some or all of the experts agree, i.e., where they share some part of their view, and identify the conflicting parts.
To this end we suggest attribute exploration of shared implications in order to examine the views of multiple domain experts.

\subsubsection{Representing an Expert's View of a Domain.}

In our setting we do not assume the existence of some true (hidden) universe; and multiple views can have contradicting information.

We define a \emph{domain} as a tuple $(\hat{G},M)$ of finite sets of objects and attributes.
A \emph{complete view} on a domain is a pair $(\hat{\K},\hat{\L})$ where $\hat{\K}=(\hat{G},M,\hat{I})$ is a formal context and $\L$ an implication theory on $M$ such that $\hat{\L}=\Imp[\hat{\K}]$.
A \emph{partial view} (or simply \emph{view}) on a domain is a pair $(\K, \L)$ where $\K=\GMWI$ is an incomplete context and $\L$ an implication theory\footnote{Note that we consider $\L$ an implication theory because it is easier to work with. In practice we can use any set of implications where the closure is satisfiable in $\K$. 
}
on $M$ such that there exists a complete view $(\hat\K,\hat\L)$ with $\K \leq \hat\K$ and $\L = \hat\L$.

In a partial view $\view$ all implications in $\L$ are satisfiable in the context $\K$, i.e., $\L \subseteq \Sat[\K]$.
Conversely, any pair of incomplete context $\K$ and implication theory $\L\subseteq\Sat[\K]$ is a partial view of the domain.

An \emph{expert $E_i$ for a domain} has a partial view $\view[i]$ of the domain where $\K_i=\GMWII[i]$.
In a group $\experts=\{E_1,\ldots,E_k\}$ of experts each expert has a partial view of the domain.
If we ask an expert $E_i$ if an implication $R\implies S$ holds in their view they answer in one of the following ways:
\begin{enumerate}
\item ``Yes, $R\implies S$ holds'', if $R\implies S\in \L_i$.
\item ``No, $R\implies S$ does not hold'', if there is a counterexample $g$ for $R\implies S$ in $\K_i$.
\item ``I do not know'', if $R\implies S$ has no counterexample in $\K_i$ and is not in $\L_i$.
\end{enumerate}
This means in particular that we assume that an expert always provides the most informative answer that is consistent with their view on the domain.

\subsubsection{Shared Implications.}

A \emph{shared implication} of a group of experts $\experts$ is an implication that holds in the view of all experts in $\experts$. More precisely:
\begin{definition}
Let $\experts=\{E_1,\ldots,E_k\}$ be a group of experts where each expert $E_i$ has a partial view $\view[i]$ of the domain.
An attribute implication $R\implies S$ \emph{holds for a group of experts} $\mathcal{F}\subseteq \experts$ if the implication holds in the view of each expert $E_i \in \mathcal{F}$, i.e., if $R\implies S \in \L_i$ for all $E_i\in \mathcal{F}$.
If $R\implies S$ holds for the experts in $\mathcal{F}$ we call it a \emph{shared implication} for $\mathcal{F}$.
\end{definition}
This is similar to the notion of \emph{conditional implications} from triadic concept analysis (c.f. \cite{FeldeStumme2020TriadicExploration,GanterObiedkov04Triadic}) when we consider each experts view to correspond to one condition in the triadic setting.


\begin{lemma}\label{lem:valid-in-subposition-iff-shared-implication}
     Let $\mathcal{F}\subseteq \experts$ be a group of experts with views $(\K_i,\L_i)$, then: 
  \begin{enumerate}
  \item The set of shared implications for a group of experts $\mathcal{F}\subseteq \experts$ is $\bigcap_{E_i\in \mathcal{F}}\L_i$. And, the implication theory of shared implications of a group of experts $\mathcal{F}$ is included in every implication theory of shared implications of subsets of $\mathcal{F}$.
  \item An implication $R\implies S$ holds in $\bigsubpos_{E_i\in \mathcal{F}}\overline\K_i$ if and only if $R\implies S$ is a shared implication for $\mathcal{F}$.
  \item For each expert $E_i$ there exists a complete view of the domain $(\hat\K_i,\L_i)$ with $\K_i\leq\hat\K_i\leq \overline\K_i$ (up to clarification of the contexts) and $\L_i=\Imp[\hat\K_i]$ and the shared implications of $\mathcal{F}$ are the implications which hold in the subposition context $\bigsubpos_{E_i\in \mathcal{F}}\hat\K_i$.

  \end{enumerate}
\end{lemma}

\begin{proof}
  1. follows directly from the definition of shared implication.
  Ad 2.: From \cref{lem:L-closed-context-has-L-implications} we know that $\Imp[\overline\K_i]=\L_i$ for each view.
  The implications that hold in $\bigsubpos_{E_i\in \mathcal{F}}\overline\K_i$ are the implications that hold in each of the contexts $\overline\K_i$, hence, $\Imp[\bigsubpos_{E_i\in \mathcal{F}}\overline\K_i]=\bigcap_{E_i\in \mathcal{F}}\L_i$.
  From 1. we know that $\bigcap_{E_i\in \mathcal{F}}\L_i$ are the shared implications for $\mathcal{F}$.
 Ad 3.: Since  $\overline\K_i(\cong (\Mod\L_i,M,\ni))$  is (up to clarification) the largest context $\tilde\K_i$ such that $\K_i \leq \tilde\K_i$ and $\Imp[\tilde\K_i]=\L_i$, we have for every complete view $(\hat\K,\hat\L)$ with $\hat\L=\L_i$ that $\K_i\leq \hat\K \leq \overline\K_i$.
\end{proof}
\subsubsection{Ordering Shared Implications.}

We utilize the notion of shared implications to hierarchically cluster the set of all attribute implications in the domain with respect to the experts in whose view they hold.
To this end, we introduce the formal context of shared implications $\C^{\x}=(\Impm,\experts,\modeledby)$ where $(R\implies S, E_i)\in \modeledby :\Leftrightarrow $ $R\implies S$ holds in the view of $E_i$.

The concepts of $\C^{\x}$ are pairs $(\L,\mathcal{F})$ consisting of a set $\L$ of implications and a set $\mathcal{F}$ of experts such that the implications in $\L$ hold in the view of all experts in $\mathcal{F}$ and $\L$ is the largest set for which this is the case.
The extents of the concepts of $\C^{\x}$ are precisely the implication theories of shared implications $\bigcap_{E_i\in F}\L_i$ for $\mathcal{F}$ where $\mathcal{F}$ is the corresponding group of experts from the intent.

The concept lattice of $\C^{\x}$ orders the shared implications with respect to the experts in whose view they hold.
We call this the \emph{system of shared implications}.

\subsection{Explore Shared Implications}\label{subsec:explore-shared-implications}
In the following we discuss the exploration of shared implications.
We begin by considering how to adapt attribute exploration to obtain the shared implications for some group of experts $\experts$.
Then, we study how to efficiently explore the shared implications for some (or all) subsets of the group $\experts$ of experts.

\subsubsection{Explore Shared Implications for a fixed Group of Experts.}
For a fixed group of experts $\experts$ for a domain the exploration algorithm to obtain the relative canonical base of their shared implications is an adapted version of attribute exploration with background information and exceptions~\cite{stumme96attribute} and triadic exploration~\cite{FeldeStumme2020TriadicExploration}.
The universe of this exploration is the formal context $\U:=\bigsubpos_{E_i\in\experts}\overline\K_i$ where $\overline \K_i$ are the respective $\L$-completions of the partial views of the experts of the domain. Note, this universe is a theoretical aid to make use of the existing theory and is dependent on the views of the experts who participate in the exploration.
In \cref{alg:explore-shared-implications} we present an implementation in pseudo-code.

Before the start of the algorithm, it is possible for the experts to provide some background information (a family $K$ of example contexts and some known shared implications $\L_0$).
We initialize $\Lurel:=\emptyset$ and $\C$ as an empty context.
Other information about the domain will be obtained by systematically asking the experts.
In each iteration the algorithm determines the next question ``Does $R\implies R^{\Box\Diamond}$ hold?'' to pose, based on the known shared implications and already provided examples.
More precisely, the premise $R$ is the next relative pseudo-intent, \ie, the lectically smallest set $R$  closed under the known shared implications and background implications that has a maximal satisfiable conclusion $R^{\Box\Diamond}$ larger than $R$ in the context of examples $\bigsubpos_{E_i\in\experts}\K_i$.
Then, each expert is asked which part of the conclusion follows from the premise and the answers are combined to determine the shared implication that holds for all experts.
The shared implication is added to the set of shared implications $\Lurel$ and used to determine the next question.
This process repeats until there is no question left, \ie, every implication can either be inferred from $\Lurel$ or the implication did not hold for some expert and a (possibly artificial) counterexample can be found in one of the contexts in $K$.%

Note, that the algorithm also logs the answers given by the experts in the context $\C$. 
This is not needed for \Cref{alg:explore-shared-implications} but will be exploited in \Cref{alg:explore-all-shared-implications} in order to prevent asking the same question multiple times when we explore the domain with multiple subsets of experts, \ie, when we want to determine the shared implications not just for one fixed group of experts, but for all possible subsets of experts.

\begin{algorithm}[!t]
  \small
  \SetKwComment{Comment}{}{}
  \SetKw{Kwin}{in} 
  \DontPrintSemicolon
  \SetAlgoLined
    \SetKwInOut{KWInteractive}{Interactive Input}
    \SetKwComment{ic}{}{}
    \KwIn{%
      %
         The set of attributes $M$ of the domain,
      a family of (possibly empty) incomplete contexts $K=\{\K_1,\ldots,\K_k\}$ containing examples given by the experts $\experts = \{E_1,\ldots,E_k\}$ and 
      a set $\L_0$ of background implications known to hold in the view of all experts  (also possibly empty)
    }
    \KWInteractive{
        $(\star)$ Each expert in $\experts$ is asked which attributes in $R^{\Box\Diamond}$ follow from $R$.
    }
    \KwOut{%
      %
The $\L_0$-relative canonical base $\Lurel$ of the shared implications, the family of (possibly enlarged) example contexts $K$ and the context of shared implications $\C$
      }
   $\Lurel:=\emptyset$\;
   $R:=\emptyset$\;
   $\C:=(\emptyset,\{E_1,\ldots,E_k\},\values,\emptyset)$\;
   \While{$R \not = M$}
   {
     \While{ $R\not = R^{\Box\Diamond}$ in $\K$ where\;
       $\K:= (G,M,W,J) := \bigsubpos_{\K_i\in K}\K_i $\;
     }{
       Ask each expert $E_i$ which attributes $m\in R^{\Box\Diamond}$ follow from $R$.\ic*[r]{$(\star)$}\label{line:a}
       For each attribute $m\in R^{\Box\Diamond}$, each expert $E_i$ can respond with
       \begin{itemize}
       \item ``Yes, $R\implies m$ holds.''
       \item ``No, $R\implies m$ does not hold'', because object $g$ is a counterexample.
         \\(which is added to $\K_i$ and thus to $K$)
       \item ``I do not know.'' (an artificial counterexample $g_{R\not\implies m}$ is added to $\K_i$)
       \end{itemize}\label{line:b}
       \vspace{-0.5\baselineskip}
       \nl\label{line:c}$S:= \{m\in R^{\Box\Diamond}| \text{ all experts responded that } R\implies m \text{ holds} \}$\;
       \nl\label{line:d}\lIf(\tcp*[f]{add shared implication}){$R\not=S$}{$\Lurel:=\Lurel\cup \{R\implies S\}$}
       \nl Extend $\C$ with $R\implies m$ for $m\in R^{\Box\Diamond}$ and the respective answers given by the experts (i.e. one of $\values$)\;
   }
   $R:= \operatorname{NextClosure}(R,M,\Lurel\cup\L_0)$ \tcc*{computes the next closure of $R$ in $M$ w.r.t. the implications $\Lurel\cup\L_0$; c.f.~\cite{GanterWille1999,ganter2016conceptual}}
 }
 \KwRet $\Lurel$, $K$ and $\C$\;
 \caption{Explore the Shared Implications of a Group of Experts}
 \label{alg:explore-shared-implications}
\end{algorithm}


\begin{theorem}\label{thm:algo-correct}
  Let $\experts=\{E_1,\ldots,E_k\}$ be a group of experts with partial views $(\K_i,\L_i)$. Let $\U=\GMI=\bigsubpos_{E_i\in\experts}\overline\K_i$ be the subposition context of the respective $\L_i$-completions of the experts' views of the domain. Then:
  \begin{enumerate}
  \item In \Cref{line:d} \Cref{alg:explore-shared-implications} accepts only implications that are valid in $\U$.\label{thm:algo-correct-accepted-implies-valid}
  \item  In \Cref{line:d} \Cref{alg:explore-shared-implications} adds a valid implication $R\implies S$ to $\Lurel$ if and only if $R$ is an $\L_0$-pseudo-intent and $S=R^{II}$.\label{thm:algo-correct-accepted-iff-L0pseudointent-and-SRII}
  \end{enumerate}
\end{theorem}

\begin{proof}

We prove \Cref{thm:algo-correct} 1. by contraposition.
Assume that $R\implies S$ is an implication that does not hold in $\U$.
If $R\implies S$ does not hold in $\U$ then there is an object $g\in \overline\K_i$ in some $\L_i$-completion $\overline\K_i$ that is a counterexample to the implication.
Since $\Imp[\overline\K_i]=\L_i$ it follows that $R\implies S\not\in \L_i$.
Hence, given the question ``Does $R\implies S$ hold?'' the expert $E_i$ does not accept the implication as valid.
Therefore, $S$ is not accepted to follow from $R$ by all experts and is not accepted as a shared implication in \Cref{line:d} of \Cref{alg:explore-shared-implications}.

    The proof of \Cref{thm:algo-correct} 2. is similar to that of \cite[Prop. 34]{ganter2016conceptual}.
  We prove this by induction over the premise size $k$ of $R$.
  We begin with the base case $R=\emptyset$.
  
  ``$\Rightarrow$'':
  An implication $R\implies S$ is added to $\Lurel$ if the conclusion $S$ is maximal and $S\not= R$, i.e., if $S=R^{II}$ and $R\not=R^{II}$.
  Hence, $\emptyset \implies S$ is added to $\Lurel$ if $S=\emptyset^{II}$ and $\emptyset\not= \emptyset^{II}$ and thus, $\emptyset$ is an $\L_0$-pseudo-intent. 
  
  ``$\Leftarrow$'':
  Now let $\emptyset$ be an $\L_0$-pseudo-intent.
  The implication $\emptyset\implies \emptyset^{II}$ holds by definition. It is added to $\Lurel$  because $\emptyset^{II}$ is the largest set for which all experts agree that it follows from $\emptyset$, and $\emptyset\not=\emptyset^{II}$ by definition of $\L_0$-pseudo-intent.

  Assume now that the proposition holds for all subsets $N$ of $M$ with  $|N|\leq k$.

  
  ``$\Rightarrow$'': 
  Let $R\implies S$ be a valid implication in $\K$ with $|R|=k+1$ and added to $\Lurel$.
  We show that $(i)$ $S=R^{II}$ and that $(ii)$ $R$ is an $\L_0$-pseudo-intent.
  Ad $(i)$: Assume that there exists $m\in R^{\Box\Diamond}\setminus R^{II}$.
  Then some expert does not confirm (with ``yes'') that the attribute follows from $R$ and a (real or artificial) counterexample is added. Thus we have $S=R^{II}$.
  Because the implication is added to $\Lurel$, we have $R\not= R^{II}$.
  Ad $(ii)$:
  Assume $R$ is not an $\L_0$-pseudo-intent. 
  Then at least one of the properties of the definition of $\L_0$-pseudo-intent does not hold.
  We show that each case yields a contradiction ($\lightning$).
  If $R$ does not respect $\L_0$ then $R$ is not suggested as premise because it is not $(\Lurel\cup\L_0)$-closed and $R\implies S$ can not be added to $\Lurel$.$\lightning$
  If $R=R^{II}$ then $R\implies R^{II}$ can not be added to $\Lurel$.$\lightning$
  If there exists an $\L_0$-pseudo-intent $P\subset R$ with $P^{II}\not\subseteq R$ then $P\implies P^{II}$ is in $\Lurel$ by induction hypothesis and because implications are added to $\Lurel$ in lectic order with respect to their premises.
  But then $R$ is not $(\Lurel\cup\L_0)$-closed and is not suggested as premise by the algorithm and thus, $R\implies S$ can not be added to $\Lurel$.$\lightning$
  Hence, $R$ is an $\L_0$-pseudo-intent.

  ``$\Leftarrow$'':
  Now let $R$ be an $\L_0$-pseudo-intent with $|R|=k+1$.
  We show that $R\implies R^{II}$ is added to $\Lurel$.
  For any implication $P\implies Q$ in $\Lurel$ with $P\subset R$ we have by the induction hypothesis that $P$ is an $\L_0$-pseudo-intent and $Q=P^{II}$.
  Because $R$ is an $\L_0$-pseudo-intent we have $P^{II}\subseteq R$.
  Hence, $R$ is $(\Lurel\cup\L_0)$-closed and the implication $R\implies R^{\Box\Diamond}$ will be suggested to the experts.
  Then $R^{II}$ is the maximal set of attributes for which all experts agree that it follows from $R$ since for every other attribute $R^{\Box\Diamond}\setminus R^{II}$ some expert has a counterexample.
  By definition of $\L_0$-pseudo-intent we have $R\not= R^{II}$ and the implication $R\implies R^{II}$ is added to $\Lurel$.
\end{proof}

\begin{corollary}
  Upon termination of \cref{alg:explore-shared-implications},
  the output $\Lurel$ is the canonical base of $\U=\bigsubpos_{E_i\in\experts}\overline\K_i$ relative to $\L_0$, and $\Lurel\cup\L_0$ is a base of $\bigcap_{E_i\in\experts}\L_i$.
  Further, $\forall R\subseteq M$ we have $R^{II}$(in $\U$) = $R^{\Box\Diamond}$ (in $\K$) at the end of the exploration.
\end{corollary}
\begin{proof}
This follows from \cref{thm:algo-correct} and \cref{lem:valid-in-subposition-iff-shared-implication}.
\end{proof}

\subsubsection{Explore the System of Shared Implications.}
Now that we know how to obtain the shared implications for some group of experts $\experts$, let us consider how to obtain the shared implication theories for some or all subsets of $\experts$.
Essentially, the goal is to obtain the concept lattice of the context of shared implications $\C^{\times}$.

From \cite{FeldeStumme2020TriadicExploration} we know that there are several viable strategies:
One option is to explore the domain for each expert separately, \ie, obtain the columns of $\C^{\times}$, then combine the results to obtain the context $\C^{\times}$ and from this compute the concept lattice $\BV(\C^{\times})$.
This has the advantage that we can parallelize the individual explorations because there is no coordination overhead during the exploration step.
However, this also means that we need all explorations to be finished before we obtain any results.
In particular, if we are only interested in the shared implications of all experts or only of a subset of $\BV(\C^{\times})$, this approach usually asks more questions than necessary.

Another option is to explore the lattice $\BV(\C^{\times})$ from bottom to top, \ie, explore the shared implications for all non-empty expert subsets of $\experts$ from largest to smallest, and reduce the amount of questions in later explorations by using the already discovered shared implications as background information.
This approach has some coordination overhead and the explorations can only be parallelized for explorations with the same number of experts.
But, it also has the advantage that only the questions necessary to find the shared implications of interest are posed.
In particular, we do not need to explore all of $\BV(\C^{\times})$ if we are only interested in some subset of the system of shared implications.
In \Cref{alg:explore-all-shared-implications} we provide an implementation of this approach in pseudo-code.

\begin{algorithm}[!t]
    \small \SetKwComment{Comment}{}{} \SetKw{Kwin}{in} 
    \DontPrintSemicolon \SetAlgoLined
    \SetKwInOut{KWInteractive}{Interactive Input}
    \SetKwComment{ic}{}{}
    \KwIn{%
A family of (possibly empty) incomplete contexts $K=\{\K_1,\ldots,\K_k\}$ containing examples given by the experts $\experts = \{E_1,\ldots,E_k\}$, a context $\C=(G_{\C},\experts,I_{\C})$ of shared implications (also possibly empty).
  }
    \KwOut{%
        A possibly enlarged family of examples $K$ and the context of shared implications $\C$ which contains the experts responses.
  
    }
\For{$\tilde \experts$ in linear extension of $(\mathcal{P}(\experts)\setminus\emptyset,\supseteq)$}{
  $\tilde K =$ subset of $K$ corresponding to the experts in $\tilde \experts$\;
  $\tilde \L = \tilde\experts^{\Box}$ (in $\C$)\;
  $\hat\L,\hat K,\hat\C = $ explore-shared-implications($M$,$\tilde K$,$\tilde\experts$,$\tilde\L$)\;
  merge $\hat K$ into K\;
  merge $\hat \C$ into $\C$ and ?-reduce where possible\;
}
\Return $\C$, $K$

 \caption{Explore the System of Shared Implications}
 \label{alg:explore-all-shared-implications}
\end{algorithm}

No matter which approach is chosen, we need to consider how to merge the results of multiple explorations with different subset of experts.
In order to merge the example contexts for each expert we use the incomplete context supremum.
Since a single expert's view is consistent for multiple explorations the example contexts for this expert do not contain any contradicting information and the supremum context exists.
The contexts of shared implications can also be merged using the incomplete context supremum.
The supremum always exists because the answers from each expert are consistent across multiple explorations.
Joining contexts of shared implications introduces some ``\q'' in the context because not all experts have responded to all questions from all explorations.
However, some of these  might be answered using other responses given by an expert; either because they follow from the accepted implications or because some given example contradicts the validity of the implication.
Therefore, after merging multiple contexts of shared implications in a second step we check if some of the newly introduced ``\q'' can be inferred from already obtained information.
 

\subsection{Improving Collaboration}
If we assume that experts for a domain want to reach a more detailed view of the domain the suggested approach of only exploring their shared views is insufficient.
However, once some shared views are explored we can use the obtained information to examine the conflicts, inconsistencies and unknowns, for example:
\begin{itemize}
\item Implications which only have artificial counterexamples.
\item Implications which are accepted by most experts and unknown to or rejected by only a few.
\item Controversial object attribute relations, i.e., relations that are accepted by some experts and rejected by others.
\item Implications and examples that conflict with each other for any two experts.
\end{itemize}
This helps a group of experts find a larger common ground and allows them to address some of the sources of disagreement.
We do not believe, however, that the exploration procedure would benefit from incorporating a conflict resolution step.
Rather, the exploration serves as a means to establish a baseline of commonalities, to make disagreements visible and to enable further cooperation.
 Discussion and potential resolution of conflicts should be a second, separate step.

\section{Conclusion and Outlook}
\label{sec:conclusion-outlook}

We have expanded on ideas for attribute exploration with multiple domain experts raised in \cite{FeldeStumme2020TriadicExploration} and provided a theoretical framework which builds on a multitude of previous works in the realm of formal concept analysis.
The resulting attribute exploration algorithm is an extension of attribute exploration that allows for multiple experts, incomplete information and background information.
 
Our approach serves as a step towards collaborative exploration for domains where experts might hold conflicting views.
An exploration of the shared views of a group of experts provides a structured approach to uncovering commonalities and differences in the experts views and can serve as a baseline for further investigations.
An example for the exploration of shared views is provided in the appendix.
As a next step, we will examine the properties of the two proposed approaches to explore $\BV(\C^{\times})$ in more detail, in particular with respect to the required expert input.

\bibliographystyle{splncs04}
\bibliography{paperbib.bib}
   
\newpage
\appendix
\section{Appendix}
\label{appendix}

\subsubsection{Example.}
In the following appendix we give an example for the exploration of shared implications.
The data set is derived from the \emph{BSI-Grundschutz-Kom\-pen\-dium}%
\footnote{\url{https://www.bsi.bund.de/SharedDocs/Downloads/EN/BSI/Grundschutz/International/bsi_it_gs_comp_2021.html}}%
, a publication by the German \emph{Federal Office for information Security}.
This publication offers standardized security requirements for typical business processes, applications, IT systems, communication links and rooms.
%
%
It contains a list of 47 potential risks, called \emph{elementary threats} (short: \emph{threats}), that are used to analyze the security requirements.
The threats (ranging from fire, water and natural disasters over eavesdropping, disclosure of sensitive information and manipulation of information to identity theft, sabotage and social engineering) will serve as our attributes.
The security requirements are organized in modules that are comprised of \emph{building blocks}.
In each building block the relevant threats are identified and a list of \emph{countermeasures} (short: \emph{measures}) is given.
The measures are different for each building block and each measure is effective against some of the elementary threats.
The measures will serve as our objects.

Hence, each building block gives rise to a formal context $\K=\GMII[i]$ of measures and elementary threats with the incidence relation being: A measure is in relation to an elementary threat if the measure is effective against the threat.
  
In our example we explore the shared implications of four building blocks, namely, APP.1.1, CON.1, ORP.1 and SYS.1.1.
Each block  is the first building block in one of the modules.
We restrict our example to these four building blocks and the attributes 18\ldots 22, cf. \Cref{table:appendix-partial-list-elementary-threats}, to keep it reasonably small.
For our example assume that we have one expert to consult for each of the building blocks.
The experts' views (consisting of the context and its implications) are presented in \Cref{fig:appendix-app,fig:appendix-con,fig:appendix-orp,fig:appendix-sys}.

An implication $R\implies S$ that holds in an expert's view can be interpreted as:
Measures effective against all threats in $R$ are also effective against all threats in $S$ (for a particular building block).

The exploration of all shared implications starts with the exploration of the shared implications for APP.1.1, CON.1, ORP.1 and SYS.1.1.
In \Cref{  appendix-exploration-full-table  } we present the questions and answers for the exploration of all shared implications.
They are arranged from top to bottom in chronological order.
The first column of \Cref{  appendix-exploration-full-table  } specifies which experts are consulted.
Each block of the table with the same experts corresponds to one exploration of their shared implications.
The second column contains the question that is being posed, and then one row for each attribute in the conclusion.
The next block of columns, APP.1.1, CON.1, ORP.1 and SYS.1.1, contains the answers given by the respective experts.
Note that there is an empty space if the expert was not consulted for this question, i.e., is not present in the first column.
The last column contains the counterexamples given when an attribute does not follow from the premise.

The result of the exploration of shared implications, more precisely, the context of shared implications $\C^{\x}$ as well as its corresponding concept lattice are presented in \Cref{fig:appendix-result}.
Note that the extents of the concepts of $\C^{\x}$ are not implication theories of shared implications but generators for them.
They consist of the implications from the relative canonical base and the respective background implications.
And, because the background implications are a union of ``stacked'' relative canonical bases the implications in the generators are complete but not necessarily irredundant.

After exploring all shared implications for the group of experts, we find, for example, that the implication $(22)\implies (21)$ does not hold for any of the views, and that $(18, 20)\implies (21)$ holds for APP.1.1 and ORP.1 but not for CON.1 and SYS.1.1.

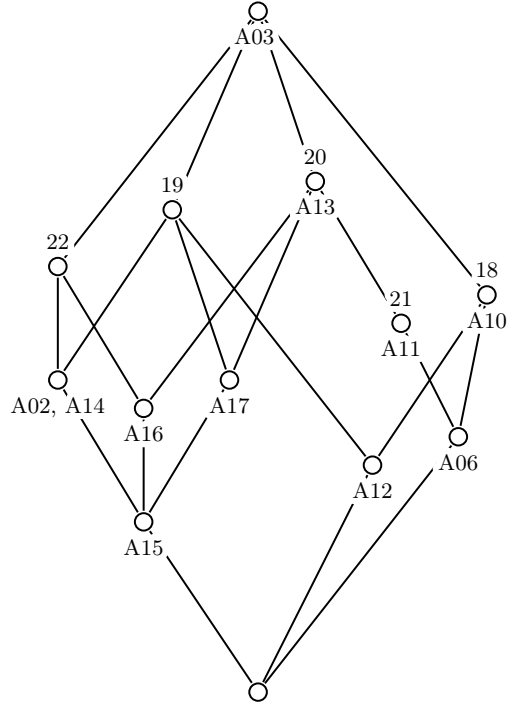
\begin{table}[t]
  \centering
  \caption{Partial list of elementary threats.}
  \begin{tabular}{l}
    \hline
  Elementary Threats\\ 
  \hline
  01 Fire\\
  02 Unfavorable Climatic Conditions\\
  \vdots\\
  18 Poor Planning or Lack of Adaptation\\
  19 Disclosure of Sensitive Information\\
  20 Information or Products from an Unreliable Source\\
  21 Manipulation with Hardware or Software\\
  22 Manipulation of Information\\
  \vdots\\
    47 Harmful Side Effects of IT-Supported Attacks\\
    \hline
\end{tabular}

\label{table:appendix-partial-list-elementary-threats}
\end{table}

\begin{figure}
  \begin{minipage}{\textwidth}
      \centering
    \resizebox{0.8\textwidth}{!}{
\includegraphics[height=10cm]{appendix/APP.1.1lattice.tikz} }
\end{minipage}

\begin{minipage}{1.0\linewidth}
\vspace{1em}
\end{minipage}

\begin{minipage}{\textwidth}
    \centering
\input{appendix/APP.1.1.cxt}
\end{minipage}
\ \\
The canonical base of APP.1.1 is:\\
\input{appendix/APP.1.1impls.tex}
  \caption{The (complete) view $(\K,\L)$ of the expert for APP.1.1.}
  \label{fig:appendix-app} 
\end{figure}

\begin{figure}
  \begin{minipage}{\textwidth}
      \centering
    \resizebox{0.8\textwidth}{!}{
\includegraphics[height=10cm]{appendix/xCON.1lattice.tikz} }
\end{minipage} 

\begin{minipage}{1.0\linewidth}
\vspace{1em}
\end{minipage}

\begin{minipage}{\textwidth}
    \centering
\input{appendix/xCON.1.cxt}
\end{minipage}
\ \\ 
The canonical base of CON.1 is:\\
\input{appendix/xCON.1impls.tex}

\caption{The (complete) view $(\K,\L)$ of the expert for CON.1.}
\label{fig:appendix-con}
\end{figure}

\begin{figure}  
  \begin{minipage}{\textwidth}
  \centering
      \resizebox{0.5\textwidth}{!}{
\includegraphics[height=10cm]{appendix/ORP.1lattice.tikz}}
\end{minipage}

\begin{minipage}{1.0\linewidth}
\vspace{1em}
\end{minipage}

\begin{minipage}{\textwidth}
    \centering
\input{appendix/ORP.1.cxt}
\end{minipage}
\ \\ 
The canonical base of ORP.1 is:\\   
\input{appendix/ORP.1impls.tex}

\caption{The (complete) view $(\K,\L)$ of the expert for ORP.1.}
\label{fig:appendix-orp}
\end{figure}

\begin{figure}
  \begin{minipage}{\textwidth} 
    \centering 
     \resizebox{0.8\textwidth}{!}{
\includegraphics[height=11cm]{appendix/SYS.1.1lattice.tikz}}
\end{minipage}

\begin{minipage}{1.0\linewidth}
\vspace{1em}
\end{minipage}
 
\begin{minipage}{\textwidth}
    \centering
\input{appendix/SYS.1.1.cxt}
\end{minipage}
\ \\
The canonical base of SYS.1.1 is:\\
\input{appendix/SYS.1.1impls.tex}

\caption{The (complete) view $(\K,\L)$ of the expert for SYS.1.1.}
\label{fig:appendix-sys}
\end{figure}


\begin{table}
  \begin{adjustwidth}{-1cm}{-1cm}
    \centering
    \caption{ Exploration of all shared implications.  }
\footnotesize
\begin{tabular}{|c|l|cccc|l|}
\hline
Experts & \parbox{4.5cm}{Implicational Question:\newline Which part of the conclusion follows from the premise?}  & \rotatebox{90}{APP.1.1} & \rotatebox{90}{CON.1} & \rotatebox{90}{ORP.1} & \rotatebox{90}{SYS.1.1} & Counterexamples \\
\hline
\hline
APP.1.1, CON.1, ORP.1, SYS.1.1 & $\emptyset$ $\rightarrow$ (18 19 20 21 22) ? &&&&& \\
 &\phantom{$\emptyset$} $\rightarrow$ (18)  & .  & .  & .  & . & A16, C11, O04, S06 \\
 &\phantom{$\emptyset$} $\rightarrow$ (19)  & .  & .  & x  & . & A16, C14, S01 \\
 &\phantom{$\emptyset$} $\rightarrow$ (20)  & .  & .  & .  & . & A02, C03, O13, S06 \\
 &\phantom{$\emptyset$} $\rightarrow$ (21)  & .  & .  & .  & . & A16, C11, O13, S35 \\
 &\phantom{$\emptyset$} $\rightarrow$ (22)  & .  & .  & .  & . & A06, C11, O16, S35 \\
\cline{2-7} & (21) $\rightarrow$ (18 20) ? &&&&& \\
 &\phantom{(21)} $\rightarrow$ (18)  & .  & .  & x  & . & A11, C16, S21 \\
 &\phantom{(21)} $\rightarrow$ (20)  & x  & .  & x  & . & C16, S21 \\
\cline{2-7} & (20 21 22) $\rightarrow$ (18 19) ? &&&&& \\
 &\phantom{(22 21 20)} $\rightarrow$ (18)  & x  & x  & x  & . & S31 \\
 &\phantom{(22 21 20)} $\rightarrow$ (19)  & x  & x  & x  & . & S31 \\
\cline{2-7} & (19 21) $\rightarrow$ (18 20 22) ? &&&&& \\
 &\phantom{(19 21)} $\rightarrow$ (18)  & x  & x  & x  & . & S34 \\
 &\phantom{(19 21)} $\rightarrow$ (20)  & x  & x  & x  & . & S34 \\
 &\phantom{(19 21)} $\rightarrow$ (22)  & x  & x  & x  & x &  \\
\cline{2-7} & (19 20 22) $\rightarrow$ (18 21) ? &&&&& \\
 &\phantom{(19 22 20)} $\rightarrow$ (18)  & .  & x  & x  & x & A15 \\
 &\phantom{(19 22 20)} $\rightarrow$ (21)  & .  & x  & x  & x & A15 \\
\cline{2-7} & (19 20 21 22) $\rightarrow$ (18) ? &&&&& \\
 &\phantom{(19 22 21 20)} $\rightarrow$ (18)  & x  & x  & x  & x &  \\
\cline{2-7} & (18 22) $\rightarrow$ (19) ? &&&&& \\
 &\phantom{(18 22)} $\rightarrow$ (19)  & x  & x  & x  & x &  \\
\cline{2-7} & (18 21) $\rightarrow$ (20) ? &&&&& \\
 &\phantom{(18 21)} $\rightarrow$ (20)  & x  & x  & x  & x &  \\
\cline{2-7} & (18 19 20 22) $\rightarrow$ (21) ? &&&&& \\
 &\phantom{(18 19 22 20)} $\rightarrow$ (21)  & x  & x  & x  & x &  \\
\hline
\hline
CON.1, ORP.1, SYS.1.1 & (20 22) $\rightarrow$ (21) ? &&&&& \\
 &\phantom{(22 20)} $\rightarrow$ (21) &  & x  & x  & x &  \\
\cline{2-7} & (20 21) $\rightarrow$ (22) ? &&&&& \\
 &\phantom{(21 20)} $\rightarrow$ (22) &  & x  & x  & x &  \\
\cline{2-7} & (18) $\rightarrow$ (19) ? &&&&& \\
 &\phantom{(18)} $\rightarrow$ (19) &  & .  & x  & . & C15, S13 \\
\cline{2-7} & (18 20) $\rightarrow$ (19) ? &&&&& \\
 &\phantom{(18 20)} $\rightarrow$ (19) &  & .  & x  & . & C09, S11 \\
\hline
\hline
APP.1.1, CON.1, SYS.1.1 & (18 19) $\rightarrow$ (20) ? &&&&& \\
 &\phantom{(18 19)} $\rightarrow$ (20)  & .  & . &  & x & A12, C07 \\
\cline{2-7} & (18 19 22) $\rightarrow$ (20 21) ? &&&&& \\
 &\phantom{(18 19 22)} $\rightarrow$ (20)  & x  & x &  & x &  \\
 &\phantom{(18 19 22)} $\rightarrow$ (21)  & x  & x &  & x &  \\
\hline
\hline
APP.1.1, CON.1, ORP.1 & (18 19 20) $\rightarrow$ (21 22) ? &&&&& \\
 &\phantom{(18 19 20)} $\rightarrow$ (21)  & x  & .  & x & & C04 \\
 &\phantom{(18 19 20)} $\rightarrow$ (22)  & x  & .  & x & & C04 \\
\hline
\hline
ORP.1, SYS.1.1 & (19 20) $\rightarrow$ (18) ? &&&&& \\
 &\phantom{(19 20)} $\rightarrow$ (18) & &  & x  & x &  \\
\hline
\hline
CON.1, ORP.1 & (21) $\rightarrow$ (22) ? &&&&& \\
 &\phantom{(21)} $\rightarrow$ (22) &  & x  & x & &  \\
\hline
\hline
APP.1.1, ORP.1 & (19 20) $\rightarrow$ (22) ? &&&&& \\
 &\phantom{(19 20)} $\rightarrow$ (22)  & . &  & x & & A17 \\
\cline{2-7} & (18 20) $\rightarrow$ (21) ? &&&&& \\
 &\phantom{(18 20)} $\rightarrow$ (21)  & x &  & x & &  \\
\hline
\hline
SYS.1.1 & (22) $\rightarrow$ (21) ? &&&&& \\
 &\phantom{(22)} $\rightarrow$ (21) & & &  & . & S02 \\
\hline
\hline
ORP.1 & $\emptyset$ $\rightarrow$ (19) ? &&&&& \\
 &\phantom{$\emptyset$} $\rightarrow$ (19) & &  & x & &  \\
\hline
\end{tabular}
\end{adjustwidth}
\label{ appendix-exploration-full-table }
\end{table}

\begin{figure}[htp]
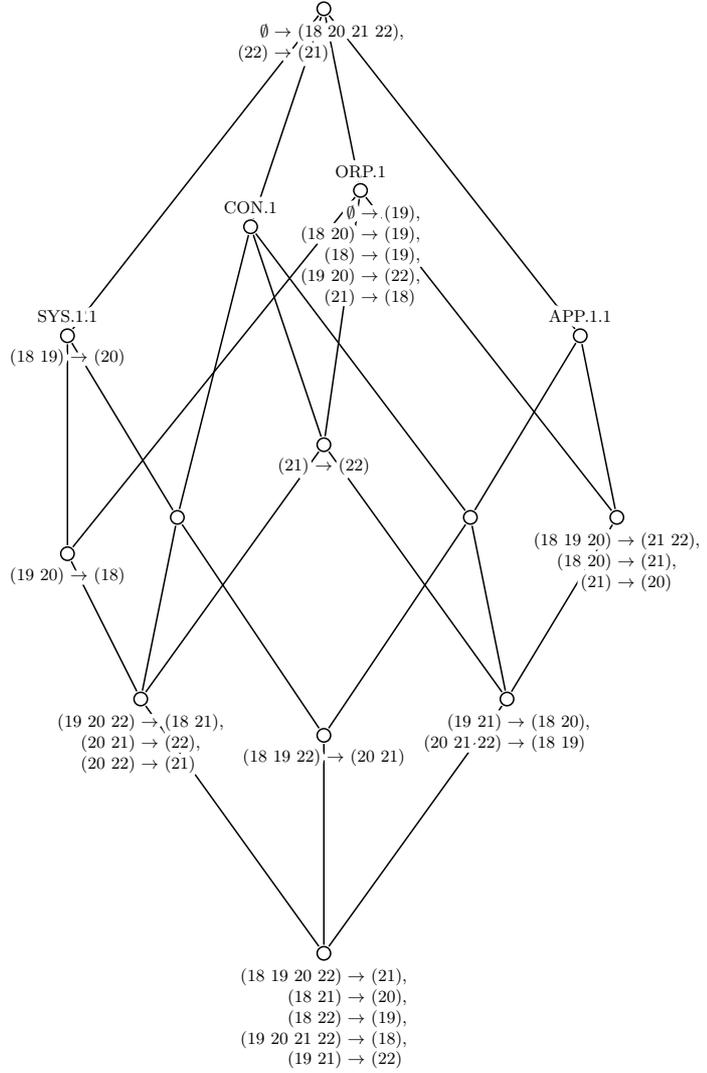

  \begin{adjustwidth}{-1cm}{-1cm}
    \begin{minipage}{0.35\textwidth}
    \centering
  \input{appendix/Explored.cxt}
\end{minipage}  
\begin{minipage}{0.65\textwidth} 
  \centering
  \resizebox{1.3\textwidth}{!}{
 \includegraphics[height=20cm]{appendix/Exploredlattice.tikz}}
\end{minipage}  
\end{adjustwidth}
 
\caption{Result of the exploration of shared implications.}
\label{fig:appendix-result}
\end{figure}

\end{document}